\documentclass{article}
\usepackage{ijcai19}

\usepackage[pdftex]{graphicx}
\graphicspath{{./graphics/}}
\usepackage{times}
\usepackage{soul}
\usepackage{url}
\usepackage[hidelinks]{hyperref}
\usepackage[utf8]{inputenc}
\usepackage[small]{caption}
\usepackage{graphicx}
\usepackage{amsmath}
\usepackage{amssymb}
\usepackage{amsthm}
\usepackage{booktabs}

\usepackage[linesnumbered,ruled]{algorithm2e}
\usepackage{algpseudocode}
\newtheorem{proposition}{Proposition}

\newcommand{\s}{\mathbf{s}}
\newcommand{\ab}{\mathbf{a}}

\title{Soft Policy Gradient Method for Maximum Entropy Deep Reinforcement Learning}

\author{
Wenjie Shi\footnote{Contact Author}\and
Shiji Song\And
Cheng Wu
\affiliations
Department of Automation, Tsinghua University, Beijing, China
\emails
shiwj16@mails.tsinghua.edu.cn,
shijis@mail.tsinghua.edu.cn,
wuc@tsinghua.edu.cn
}

\begin{document}

\maketitle

\begin{abstract}
Maximum entropy deep reinforcement learning (RL) methods have been demonstrated on a range of challenging continuous tasks. However, existing methods either suffer from severe instability when training on large off-policy data or cannot scale to tasks with very high state and action dimensionality such as 3D humanoid locomotion. Besides, the optimality of desired Boltzmann policy set for non-optimal soft value function is not persuasive enough. In this paper, we first derive soft policy gradient based on entropy regularized expected reward objective for RL with continuous actions. Then, we present an off-policy actor-critic, model-free maximum entropy deep RL algorithm called deep soft policy gradient (DSPG) by combining soft policy gradient with soft Bellman equation. To ensure stable learning while eliminating the need of two separate critics for soft value functions, we leverage double sampling approach to making the soft Bellman equation tractable. The experimental results demonstrate that our method outperforms in performance over off-policy prior methods.
\end{abstract}

\section{Introduction}
Model-free reinforcement learning (RL) aims to acquire an effective behavior policy through ongoing trial and error interaction with a black box environment. The sole goal of standard RL is to optimize the quality of an agent's behavior policy by maximizing cumulative discounted rewards. And so far, standard model-free RL has been applied to a range of challenging domains, such as games \cite{Silver2016Mastering}, robotics \cite{Levine2016End}, finance \cite{prashanth2016cumulative} and healthcare \cite{shortreed2011informing}. However, some notorious drawbacks of standard model-free RL such as sample complexity, hyperparameter sensitivity or instability, have limited its widespread adoption in real-world domains.

Generally, policy optimization and value iteration are two basic paradigms in standard model-free RL. Policy based approaches, such as TRPO \cite{schulman2015trust}, directly optimize the quantity of interest while remaining stable under function approximation, but policy based RL typically requires on-policy learning that is extravagantly expensive. By contrast, value based approaches, such as DQN \cite{mnih2015human}, can learn from any "off-policy" transitions sampled from identical environment, making them inherently more sample efficient \cite{gu2017q}. But unfortunately, off-policy learning does not stably interact with function approximation. Recently, some attempts have been made by combining policy and value based RL under actor-critic framework, such as DDPG \cite{lillicrap2015continuous}, but there still remain some unsettled issues including how to set choose the type of policy and exploration noise.

Maximum entropy RL, which augments the standard maximum RL objective with an entropy regularization term \cite{toussaint2009robot}, is another popular framework that can handle the tasks with multi-modality. As discussed in prior work, a stochastic policy may emerge as the optimal answer when maximum entropy RL tries to connect the optimal control with probabilistic inference \cite{todorov2008general}. Intuitively, framing control as inference produces policies that try to learn all of the ways instead of the best way to perform the task. More importantly, maximum entropy framework shows several amazing advantages when compared to standard RL framework. First, entropy regularization encourages exploration and helps prevent early convergence to sub-optimal policies. Second, the resulting policies can serve as a good initialization for finetuning to a more specific behavior. Third, maximum entropy framework provides a better exploration mechanism for seeking out the best mode in a multimodal reward landscape. Fourth, the resulting policies are more robust in the face of adversarial perturbations as demonstrated in \cite{haarnoja2017reinforcement}.

Actually, prior works have proposed various policy and value based RL methods under maximum entropy framework, such as combining policy gradient with Q-learning \cite{o2016combining}, soft Q-learning \cite{haarnoja2017reinforcement,schulman2017equivalence}. However, policy based methods requiring on-policy learning suffer from poor sample complexity, while value based methods applying off-policy learning need a complex approximate sampling procedure in continuous action spaces.

In this paper, we explore to design an off-policy and stable model-free deep RL algorithm by combining policy and value based methods under maximum entropy RL framework, which we call Deep Soft Policy Gradient (DSPG). First, we propose soft policy gradient under maximum entropy RL framework, and rigorous derivation of this proposition is given. Second, soft policy gradient is combined with soft Bellman equation by employing two deep function approximators to learn the soft Q-function and the stochastic policy, respectively. Finally, we present experimental results that show a significant improvement in performance over prior methods including DDPG and SAC.

\section{Related Work}
Our deep soft policy gradient (DSPG) algorithm mainly incorporates two ingredients: an off-policy actor-critic architecture with separate stochastic policy and action-value function networks, and an entropy regularization term to enable stability and exploration. We review prior works that draw on some of these ideas in this section.

Off-policy actor-critic has been verified as a feasible method to improve sample efficiency by exploiting off-policy data from other sources, such as past experience. And a particularly popular implementation of off-policy actor-critic architecture is DDPG \cite{lillicrap2015continuous}, which employs a Q-function estimator to enable off-policy learning, and a deterministic actor to optimize this Q-function by applying deterministic policy gradient \cite{silver2014deterministic}. However, some issues are still not completely settled. First, for many tasks with continuous action spaces, it is not possible to straightforwardly apply Q-learning which solves the greedy policy, DDPG instead uses deterministic policy to avoid global optimization with respect to the action at every timestep, but which is at the cost of performance loss. Second, additional noise process is required for exploration in the context of deterministic policy, the type and scale of noise must be chosen meticulously for different problem settings to achieve good performance. Furthermore, the interplay between the deterministic actor and the Q-function typically makes DDPG suffer from severe hyperparameter sensitivity \cite{henderson2017deep}. As a consequence, it is difficult to generalize DDPG to very complex, high-dimensional tasks.

Maximum entropy reinforcement learning optimizes policies to maximize the entropy regularized expected reward objective. More recently, several progresses have been made in developing policy and value based methods under the framework of maximum entropy reinforcement learning. Roughly speaking, existing works fall into two categories: softmax temporal consistency \cite{nachum2017bridging,nachum2017trust} and soft policy iteration \cite{o2016combining,haarnoja2017reinforcement,haarnoja2018soft}. Softmax temporal consistency between the optimal policy and soft optimal state value leads to path-wise consistency learning (PCL) methods \cite{nachum2017bridging} and Trust-PCL \cite{nachum2017trust}, which solve jointly the policy and soft state value by minimizing soft consistency error. However, PCL and Trust-PCL depend on complete trajectories and succumb to the instability when training on large off-policy data. Concurrent to PCL and Trust-PCL, some other methods such as PGQ \cite{o2016combining} and soft Q-learning \cite{haarnoja2017reinforcement}, instead apply the idea underlying generalized policy iteration to learn maximum entropy policies by alternating policy evaluation and policy improvement. However, PGQ operate on simple tabular representations and are difficult to scale to continuous or high-dimensional domains, while soft Q-learning draws samples from an approximate sampling network. Building on soft Q-learning, soft actor-critic (SAC) \cite{haarnoja2018soft} realizes policy improvement by minimizing Kullback-Leibler divergence between the current policy and the desired policy. However, how to choose the desired policy set for non-optimal value functions is somewhat subjective. Moreover, a separate function approximator for soft state-value is necessary to stabilize the learning, but which results in cumulative approximation errors. In contrast, we derive directly soft policy gradient based on the entropy regularized expected reward objective. And double sampling approach is utilized to stabilize the learning without separate soft value function approximator.

\section{Preliminaries}
In this section, we will define the reinforcement learning problem addressed in this paper and briefly summarize maximum entropy reinforcement learning framework.

\subsection{Notation}
We study reinforcement learning and control tasks with continuous action spaces. An agent aims to learn the optimal policies to maximize an entropy regularized expected reward objective by continuing trial-and-error in a stochastic environment. To this end, we define the problem as policy search in an infinite-horizon Markov decision process (MDP), which consists of the tuple $(\mathcal{S},\mathcal{A},p,r)$. The state space $\mathcal{S}$ and action space $\mathcal{A}$ are assumed to be continuous, and the state transition probability $p^{\ab_t}_{\s_t\s_{t+1}}: \mathcal{S}\times\mathcal{S} \times\mathcal{A}\rightarrow[0,\infty)$ represents the probability density of next state $\s_{t+1}\in\mathcal{S}$ given current state $\s_t\in\mathcal{S}$ and action $\ab_t\in\mathcal{A}$. The environment emits a reward $r:\mathcal{S}\times\mathcal{A}\rightarrow[r_{\rm min},r_{\rm max}]$ on each transition, which we will abbreviate as $r_t\triangleq r(\s_t,\ab_t)$ to simplify notation. Throughout this paper, we use $p_ k^\pi(\s\rightarrow\s')$ to denote the probability of going from state $\s$ to state $\s'$ in $k$ steps under the policy $\pi$. For $\s_0=s,\s_k=\s'$, we have
\begin{align}\label{eqn:trajectory probability}
 \begin{split}
  p_k^\pi(\s\rightarrow\s')=&\sum\nolimits_{\ab_0}\pi(\ab_0|\s_0)
                             \sum\nolimits_{\s_1}p^{\ab_0}_{\s_0\s_1}\cdots \\
                            &\sum\nolimits_{\ab_{k-1}}\pi(\ab_{k-1}|\s_{k-1})
                             p^{\ab_{k-1}}_{\s_{k-1}\s_k},
 \end{split}
\end{align}
except that $p_0^\pi(\s\rightarrow\s)=1$. Besides, we will also use $\rho_\pi(\s_t)$ and $\rho_\pi(\s_t,\ab_t)$ to denote the state and state-action marginals of trajectory distribution induced by the policy $\pi(\ab_t|\s_t)$.

\subsection{Maximum Entropy Reinforcement Learning}
Different from standard reinforcement learning, the sole goal of maximum entropy reinforcement learning is to maximize a more general entropy regularized objective, which is defined by augmenting the expected sum of rewards with discounted entropy terms, such that the optimal policy aims to maximize its entropy at each visited state:
\begin{align}\label{eqn:objective}
 J(\pi) = \sum\nolimits_{t=0}^{T-1}\mathbb{E}_{(\s_t,\ab_t)\sim \rho_\pi}\big[r(\s_t,\ab_t)+\tau\mathcal{H}^\pi(\cdot|\s_t)\big],
\end{align}
where $\mathcal{H}^\pi(\cdot|\s_t)=-\sum_{\ab_t}\pi(\ab_t|\s_t)\log\pi(\ab_t|\s_t)$ is the entropy of policy $\pi$ at state $\s_t$. And $\tau\geq 0$ is a user-specified temperature parameter that controls the degree of entropy regularization. For the rest of this paper, we omit writing the temperature explicitly, as it can always be subsumed into the reward by scaling it by $\tau^{-1}$. Furthermore, this objective can be easily extended to infinite horizon problems by introducing a discount factor $\gamma$ to ensure that the expected sum of rewards and entropies is finite.

Under the framework of maximum entropy reinforcement learning, soft Q-function $Q^\pi(\s_t,\ab_t)$ describes the expected sum of discounted future rewards and entropies except the entropy of state $\s_t$ after taking an action $\ab_t$ in state $\s_t$ and thereafter following the policy $\pi$:
\begin{align}\label{eqn:soft Q-function}
 \begin{split}
  Q^\pi(\s_t,\ab_t) =& r_t+   \\
  &\sum\nolimits_{i>t}\gamma^{i-t}\mathbb{E}_{(\s_i,\ab_i)\sim \rho_\pi}\big[r_i+ \mathcal{H}^\pi(\cdot|\s_i)\big].
 \end{split}
\end{align}
Similar to standard reinforcement learning, we can derive directly soft Bellman equation from the above definition of soft Q-function as follows.
\begin{align}\label{eqn:soft Bellman equation}
 \begin{split}
  Q^\pi(\s_t,\ab_t) &= \mathbb{E}_{\s_{t+1}\sim p}\big[r_t +    \\
  &\gamma\mathbb{E}_{\ab_{t+1}\sim\pi}[Q^\pi(\s_{t+1},\ab_{t+1})+ \mathcal{H}^\pi(\cdot|\s_{t+1})]\big].
 \end{split}
\end{align}
Prior works have made great progresses in learning maximum entropy policy by directly solving for optimal soft Q-function \cite{ziebart2008maximum,fox2015taming,haarnoja2017reinforcement}. In subsequent sections, we will first propose and prove rigorously soft policy gradient under the framework of maximum entropy reinforcement learning, which corresponds to deterministic policy gradient in standard reinforcement learning. Then, we will further discuss how we can combine soft policy gradient with soft Bellman equation to develop an off-policy and stable deep reinforcement learning algorithm called deep soft policy gradient (DSPG), which corresponds to DDPG in standard reinforcement learning.

\section{Soft Policy Gradient}
Inspired by the observation that policy gradient methods \cite{degris2012model} are perhaps the most widely used for tasks with continuous action spaces in standard reinforcement learning, we aim to develop corresponding soft policy gradient methods for maximum entropy reinforcement learning. In this section, we derive directly soft policy gradient based on the entropy regularized expected reward objective and thereafter combine it with function approximation.

\subsection{Soft Policy Gradient}
The basic idea behind soft policy gradient is to represent the policy by a parametric probability distribution $\pi(\ab|\s)=\mathbb{P}[\ab|\s;\theta]$ that stochastically selects action $\ab$ in state $\s$ according to parameter vector $\theta$. And soft policy gradient typically proceed by sampling this stochastic policy and adjusting the parameter $\theta$ of this policy in the direction of greater entropy regularized expected reward objective.
\begin{proposition}[Soft Policy Gradient]
 Suppose that the MDP satisfies the policy $\pi(\ab|\s)$ is differentiable with respect to its parameter $\theta$, i.e., that $\frac{\partial\pi(\ab|\s)}{\partial\theta}$ exists. Then,
 \begin{align}\label{eqn:soft policy gradient}
  \begin{split}
   \nabla_{\theta}J(\pi_\theta) =& \mathbb{E}_{(\s,\ab)\sim\rho_{\pi_\theta}}\big[\big(Q^\pi(\s,\ab) \\
   &-\log\pi(\ab|\s)-1\big)\nabla_{\theta}\log\pi(\ab|\s)\big].
  \end{split}
 \end{align}
\end{proposition}

\begin{proof}
According to the definition of soft Q-function (\ref{eqn:soft Q-function}), we rewrite the entropy regularized expected reward objective (\ref{eqn:objective}) as follows.
\begin{align}\label{eqn:Q based objective}
 J(\pi_\theta) = \mathbb{E}_{(\s,\ab)\sim\rho_{\pi_\theta}} \big[Q^\pi(\s,\ab)+\mathcal{H}^\pi(\cdot|\s)\big].
\end{align}
To get the gradient of (\ref{eqn:Q based objective}), we first derive the gradient of the entropy regularized expected reward objective $J(\pi_\theta|\s_0)$ with a designated start state $\s_0$.

For the convenience of displaying, we give the following simplified representation.
\begin{align}
 \mathcal{G}(\pi_\theta|\s) = \sum_{\ab}\frac{\partial \pi(\ab|\s)}{\partial\theta}Q^\pi(\s,\ab)+\frac{\partial \mathcal{H}^\pi(\cdot|\s)}{\partial\theta}. \label{eqn:proof-0}
\end{align}
Then, we have
\begin{align}
 \nonumber
 \frac{\partial J(\pi_\theta|\s_0)}{\partial\theta} &=\frac{\partial}{\partial\theta}\mathbb{E}_{\ab_0\sim\pi_\theta} \big[Q^\pi(\s_0,\ab_0)+\mathcal{H}^\pi(\cdot|\s_0)\big]  \\
 \nonumber
 &= \frac{\partial}{\partial\theta}\bigg[\sum_{\ab_0}\pi(\ab_0|\s_0)Q^\pi(\s_0,\ab_0) +\mathcal{H}^\pi(\cdot|\s_0)\bigg] \\
 &= \sum_{\ab_0}\pi(\ab_0|\s_0)\frac{\partial Q^\pi(\s_0,\ab_0)}{\partial\theta} +\mathcal{G}(\pi_\theta|\s_0)\label{eqn:proof-1}.
\end{align}
Substituting the soft Bellman equation (\ref{eqn:soft Bellman equation}) into (\ref{eqn:proof-1}), we obtain
\begin{align}
 \begin{split}
 \frac{\partial J(\pi_\theta|\s_0)}{\partial\theta} &= \gamma p_1^\pi(\s_0\rightarrow\s_1)\sum_{\ab_1}\pi(\ab_1|\s_1) \frac{\partial Q^\pi(\s_1,\ab_1)}{\partial\theta} \label{eqn:proof-2}\\
 &~~~~~+\sum_{k\in\{0,1\}}\sum_{\s_k}\gamma^k p_k^\pi(\s_0\rightarrow\s_k) \mathcal{G}(\pi_\theta|\s_k).
 \end{split}
\end{align}
Repeating the above expansion step infinite times, we get
\begin{align}
 \nonumber
 \frac{\partial J(\pi_\theta|\s_0)}{\partial\theta} &= \sum_{k=0}^{\infty}\sum_{\s_k}\gamma^k p_k^\pi(\s_0\rightarrow\s_k)\mathcal{G}(\pi_\theta|\s_k) \label{eqn:proof-3} \\
 &= \sum_{\s}\sum_{k=0}^{\infty}\gamma^k p_k^\pi(\s_0\rightarrow\s)\mathcal{G}(\pi_\theta|\s).
\end{align}
Then, we can get the gradient of (\ref{eqn:Q based objective}) by calculating the expectation of start state $\s_0$ in (\ref{eqn:proof-3}) as follows:
\begin{align}
 \nonumber
 \frac{\partial J(\pi_\theta)}{\partial\theta} \!&= \sum_{\s_0}\rho_{\pi}(\s_0)\sum_{\s}\sum_{k=0}^{\infty}\gamma^k p_k^\pi(\s_0\rightarrow\s)\mathcal{G}(\pi_\theta|\s) \\
 \nonumber
 &= \sum_{\s}\sum_{\s_0}\rho_{\pi}(\s_0)\sum_{k=0}^{\infty}\gamma^k p_k^\pi(\s_0\rightarrow\s)\mathcal{G}(\pi_\theta|\s) \\
 \nonumber
 &= \sum_{\s}\rho_{\pi}(\s)\bigg[\sum_{\ab}\frac{\partial \pi(\ab|\s)}{\partial\theta}Q^\pi(\s,\ab)+\frac{\partial \mathcal{H}^\pi(\cdot|\s)}{\partial\theta}\bigg] \\
 \nonumber
 &= \sum_{\s}\!\rho_{\pi}(\s)\!\sum_{\ab}\!\big(Q^\pi\!(\s,\ab)\!-\!\log\pi(\ab|\s)\!-\!1\big)\!\frac{\partial \pi(\ab|\s)}{\partial\theta} \\
 \nonumber
 &= \mathbb{E}_{(\s,\ab)\sim\rho_{\pi_\theta}}\big[\big(Q^\pi(\s,\ab) \\
 \nonumber
 &~~~~~~~~~~~~~~~~~~~~~~~~-\log\pi(\ab|\s)-1\big)\nabla_{\theta}\log\pi(\ab|\s)\big].
\end{align}
\end{proof}

The above expression of soft policy gradient is surprisingly simple. In particular, there are no terms of the form $\frac{\partial \rho_{\pi_\theta}(\s)}{\partial \theta}$, which means the effect of policy changes on the distribution of states does not appear.

Moreover, this proposition has important practical value, because it is convenient for approximating the gradient by sampling. Just as policy gradient theorem \cite{sutton2000policy} in standard reinforcement learning, a key issue that soft policy gradient must address is how to estimate the soft Q-function $Q^\pi(\s,\ab)$. Actually, a simple approach is to use a sampled return, which has led to a variant of the REINFORCE algorithm \cite{williams1992simple}.

\subsection{Soft Policy Gradient with Approximation}
As discussed above, using the true return to approximate soft Q-function $Q^\pi(\s,\ab)$ is a feasible approach, but which is notoriously expensive in terms of sample complexity. Instead we consider the case in which $Q^\pi(\s,\ab)$ is approximated by a learned function approximator. In principle, if the approximation is sufficiently good, we might hope to use it in place of $Q^\pi(\s,\ab)$ in (\ref{eqn:soft policy gradient}) and still point roughly in the direction of the gradient. And \cite{singh1994learning} has proved that for the special case of function approximation arising in a tabular POMDP one could assure positive inner product with the gradient, which is sufficient to ensure improvement for moving in that direction.

However, substituting a function approximator $Q^\omega(\s,\ab)$ for the true soft Q-function $Q^\pi(\s,\ab)$ may introduce bias. Fortunately, there will be no bias if the following two conditions are satisfied according to \cite{sutton2000policy}.
\begin{itemize}
 \item The function approximator $Q^\omega(\s,\ab)$ is compatible in the sense that
 \begin{align}\label{eqn:compatible function}
  \frac{\partial Q^\omega(\s,\ab)}{\partial\omega}=\frac{\partial\pi(\s,\ab)}{\partial\theta}\frac{1}{\pi(\s,\ab)}.
 \end{align}
 \item The parameters $\omega$ are chosen to minimise the mean-squared error $\varepsilon^2(\omega)$.
 \begin{align}\label{eqn:mean-squared error}
  \varepsilon^2(\omega)=\mathbb{E}_{(\s,\ab)\sim\rho_{\pi_\theta}}\bigg[\big(Q^\omega(\s,\ab)-Q^\pi(\s,\ab) \big)^2\bigg].
 \end{align}
\end{itemize}
Then, we have the following soft policy gradient with function approximation.
\begin{align}\label{eqn:soft policy gradient with approximation}
  \nabla_{\theta}J(\pi_\theta) = \mathbb{E}_{(\s,\ab)\sim\rho_{\pi_\theta}}&\big[\big(Q^\omega(\s,\ab) \\
  &-\log\pi(\ab|\s)-1\big)\nabla_{\theta}\log\pi(\ab|\s)\big]. \nonumber
\end{align}

In other words, compatible function approximators are linear in the same features as the stochastic policy, and the parameters $\omega$ are the solution to the linear regression problem that estimates $Q^\pi(\s,\ab)$ from these features. In practice, condition (\ref{eqn:mean-squared error}) is usually relaxed in favour of policy evaluation algorithms that estimate the value function more efficiently by temporal-difference learning \cite{bhatnagar2008incremental}. Based on (\ref{eqn:soft policy gradient with approximation}), a form of policy iteration with function approximation can be proved to converge to a locally maximum entropy policy according to prior works \cite{bertsekas1996neuro,sutton2000policy}.

\section{Combined with Soft Bellman Equation}
In this section, we will present our proposed deep soft policy gradient (DSPG) algorithm, which is motivated by combining soft policy gradient with soft Bellman equation. First, We will further discuss how to achieve soft policy evaluation with soft Bellman equation (\ref{eqn:soft Bellman equation}) when a function approximator is used for soft Q-function. Then, the details of our DSPG algorithm will be given.

\subsection{Soft Bellman Equation with Approximation}
To improve the sample efficiency, we employ a function approximator $Q^\omega(\s,\ab)$ to estimate soft Q-function $Q^\pi(\s,\ab)$ in order to make soft policy gradient tractable while avoiding on-policy interaction with environment. Similar to many approaches in standard reinforcement learning, we make use of soft Bellman equation (\ref{eqn:soft Bellman equation}) to update the parameters $\omega$ of function approximator $Q^\omega(\s,\ab)$.

According to soft Bellman equation (\ref{eqn:soft Bellman equation}), we first define soft Bellman backup operator $\mathcal{T}Q^\omega$ as follows.
\begin{align}\label{eqn:soft Bellman backup operator}
 \begin{split}
  \mathcal{T}Q^\omega =&~r_t + \\
   &\gamma\mathbb{E}_{\ab_{t+1}\sim\pi}\big[Q^\omega(\s_{t+1},\ab_{t+1})+ \mathcal{H}^\pi(\cdot|\s_{t+1})\big].
 \end{split}
\end{align}
But in practice, we typically use the following sampled soft Bellman backup operator.
\begin{align}\label{eqn:sampled soft Bellman backup operator}
 \tilde{\mathcal{T}}Q^\omega &= r_t + \\  \nonumber
 &\frac{\gamma}{M}\sum\nolimits_{j}\big[Q^\omega(\s_{t+1},\ab_{(t+1)j})- \log\pi_{\theta^\prime}(\ab_{(t+1)j}|\s_{t+1})\big].
\end{align}
where the actions $\ab_{(t+1)j}$ are sampled from the target policy $\pi_{\theta^\prime}(\cdot|\s_{t+1})$ of current policy. Different from DDPG, which is a commonly used off-policy actor-critic algorithm in standard reinforcement learning and applies the deterministic policy to avoid the inner expectation in Bellman equation, we approximate the expectation of action in (\ref{eqn:soft Bellman backup operator}) by repeatedly sampling with the current policy. Indeed, the deterministic policy may induce catastrophic and unstable update of the policy and value function.

Note that the outer expectation in soft Bellman equation (\ref{eqn:soft Bellman equation}) depends only on the environment or state transition probability, which makes it possible to learn $Q^\omega(\s,\ab)$ by using previous off-policy transitions sampled from a replay buffer. Therefore, the parameters $\omega$ of $Q^\omega(\s,\ab)$ can be optimized by minimizing the soft loss:
\begin{align}\label{eqn:soft loss}
 \mathcal{L}(\omega) = \frac{1}{N}\sum_{i}\bigg[\big(Q^\omega(\s_i,\ab_i)- \tilde{\mathcal{T}}Q^{\omega^\prime}\big)^2\bigg].
\end{align}
where $Q^{\omega^\prime}$ is the target network used to stabilize the learning of soft Q-function $Q^\omega$.

\subsection{Deep Soft Policy Gradient Algorithm}
It is not possible to straightforwardly apply soft Bellman equation (\ref{eqn:soft Bellman equation}) to continuous action spaces, because the inner expectation with respect to the action is intractable when there is not an explicitly feasible way to sample from the current policy. Instead, we used an off-policy actor-critic approach based on soft policy gradient derived in last section. Specifically, as discussed above, soft policy gradient method maintains a parameterized actor function $\pi_\theta(\ab|\s)$, which specifies the current policy by stochastically mapping states to all possible actions, while the critic $Q^\omega(\s,\ab)$ is learned to evaluate the current policy by optimizing the soft loss (\ref{eqn:soft loss}) based on soft Bellman equation.

\paragraph{Gradient Clipping.} One unique challenge in maximum entropy reinforcement learning when applying soft policy gradient is that there exist the terms $Q^\omega(\s,\ab)$ explicitly in (\ref{eqn:soft policy gradient with approximation}), which means that soft Q-function will directly influence the scale of gradient and thus lead to unstable update of policy and even bad convergence. Motivated by the idea to deal with gradient explosion problem, we leverage gradient clipping approach to address this issue.
Gradient clipping approach aims to limit soft policy gradient to an appropriate range by a global norm $\mathcal{N}$. Intuitively, clipped gradients prevent drastic updates of the current policy and thus stabilize the learning. And an additional benefit from gradient clipping is that constrained soft policy gradient eliminates the impact of different baselines of soft Q-functions in various reinforcement learning tasks. By contrast, it is not necessary for DDPG to use gradient clipping approach, because the deterministic policy gradient depends on $\frac{\partial Q^\omega(\s,\ab)}{\partial \ab}$ rather than $Q^\omega(\s,\ab)$.

\paragraph{Double Sampling.} Another key challenge is that soft policy gradient (\ref{eqn:soft policy gradient with approximation}) depends on not only state distribution but also the current policy, which makes it impossible to compute soft policy gradient with transitions sampled from replay buffer. Instead we utilize double sampling approach to estimate the expectation with respect to action.
The main idea underlying double sampling approach is to separate the expectation with respect to the state and action by two sampling steps. The first sampling step uses the transitions sampled from replay buffer to approximate the expectation with respect to state distribution induced by the current policy, which is commonly used in policy gradient implementations, such as DPG and DDPG. In the second sampling step, we estimate the expectation with respect to action by sampling with the current policy. Although existing off-policy actor-critic algorithm \cite{degris2012model} uses a different behaviour policy to generate trajectories, double sampling approach has distinct advantage in terms of performance and sampling complexity.

Based on above gradient clipping and double sampling approaches, soft policy gradient can be rewritten in a tractable form as follows.
\begin{align}
 \label{eqn:tractable soft policy gradient}
 \nabla_{\theta}J(\pi_\theta) &\approx \frac{1}{NM}\sum\nolimits_{i}\sum\nolimits_{j}\big[\big(Q^\omega(\s_i,\ab_{ij}) \\
 \nonumber
 &~~~~~-\log\pi(\ab_{ij}|\s_i)-1\big)\nabla_{\theta}\log\pi(\ab_{ij}|\s_i)\big]. \\
 \label{eqn:clipped soft policy gradient}
 \nabla_{\theta}^{CLIP}J(\pi_\theta) &= {Clip\_by\_norm}(\nabla_{\theta}J(\pi_\theta), \mathcal{N})
\end{align}
where $\ab_{ij}$ is sampled with the current policy $\pi(\cdot|\s_i)$. ${Clip\_by\_norm}$ denotes the gradient clipping operator.

To summarize, we propose deep soft policy gradient (DSPG, Algorithm \ref{alg:DSPG}) algorithm for learning maximum entropy policies in continuous domains. The algorithm proceeds by alternating between evaluating the current policy with soft Bellman equation and improving the current policy with soft policy gradient.

\begin{algorithm}[!t]
 \caption{DSPG Algorithm}
 \label{alg:DSPG}
 Randomly initialize the critic $Q^\omega(\s,\ab)$ and actor $\mu_\theta (\cdot|\s)$ with weights $\omega$ and $\theta$\;
 Initialize the target network $Q^{\omega^\prime}$ and $\mu_{\theta^\prime}$ with $\omega^\prime\leftarrow\omega$ and $\theta^\prime\leftarrow\theta$\;
 Initialize replay buffer $\mathcal{R}$\;
 \For{each episode}
 {
  Reset initial observation state $\s_0$\;
  \For{each step}
  {
   Sample an action $\ab_t$ for $\s_t$ using $\pi_\theta(\cdot|\s_t)$\;
   Execute action $\ab_t$ and receive $r_t$, $\s_{t+1}$\;
   Store transition ($\s_t,\ab_t,r_t,\s_{t+1}$) in $\mathcal{R}$\;
   Sample a random minibatch of $N$ transitions ($\s_i,\ab_i,r_i,\s_{i+1}$) from $\mathcal{R}$\;
   Sample $M$ actions $\ab_{(i+1)j}$ for each state $\s_{i+1}$ using $\pi_{\theta^\prime}(\cdot|\s_{i+1})$\;
   Compute the soft loss $\mathcal{L}(\omega)$ (\ref{eqn:soft loss})\;
   Update the critic by minimizing the soft loss\;
   Sample $M$ actions $\ab_{ij}$ for each state $\s_i$ using $\pi_\theta(\cdot|\s_i)$\;
   Compute the clipped soft policy gradient (\ref{eqn:clipped soft policy gradient})\;
   Update the actor by applying the clipped soft policy gradient $\nabla_{\theta}^{CLIP}J(\pi_\theta)$\;
   Update the target networks:
   $~~~~~~~~~~~~~~~~~~~~~\omega^\prime\leftarrow\alpha\omega+(1-\alpha)\omega^\prime$
   $~~~~~~~~~~~~~~~~~~~~~~\theta^\prime\leftarrow\alpha\theta+(1-\alpha)\theta^\prime$
   }
  }
\end{algorithm}

\section{Experiments}
We evaluate our proposed algorithm, namely DSPG, across several benchmark continuous control tasks and compare them to standard DDPG and SAC implementations. We find that DSPG can consistently match or beat the performance of these baselines.

\subsection{Setup}
We chose four well-known benchmark continuous control tasks (Ant, Hopper, HalfCheetah and Walker2d) available from OpenAI Gym and utilizing MuJoCo environment. And two feed-forward neural networks were used to represent the policy and value estimates. In addition, we used uniformly a simple policy represented by a multivariate gaussian for both DSPG and SAC.

For DDPG, identical experimental setup was employed in our DDPG as in \cite{lillicrap2015continuous}, actually which has to perform a grid search to achieve good performance due to its hyperparameter sensitivity. Besides, to be consistent with DDPG and DSPG, we merely compared to the general form of SAC without applying additional tricks, such as hard target update every fixed interval and using two Q-functions, although these tricks are effective for all three algorithms.

To evaluate DSPG's performance with an appropriate global gradient norm $\mathcal{N}$, we performed a simple grid search over $\mathcal{N}\in\{1,3,5\}$. And the search results show that smaller global gradient norm will slightly slow down the learning but make the learning more stable. Based on this conclusion, we chose $\mathcal{N}=5$ for both Ant and Walker2d, while $\mathcal{N}=1$ and $\mathcal{N}=3$ were chosen for Hopper and HalfCheetah, respectively. Moreover, we sampled $M=64$ times with the current learned policy $\pi_\theta$ to estimate the expectation with respect to the action in each train step, and we performed 4 train steps every interaction step to accelerate the learning for all algorithms. A detailed description of our implementation and experimental setup is available in the Appendix \ref{app:setup}. The source code of our DSPG implementation will be available online after the paper is accepted.

\subsection{Results}
\begin{figure}[!t]
 \centering
 \scriptsize
  \includegraphics*[width=0.95\linewidth,viewport=10 195 580 635]{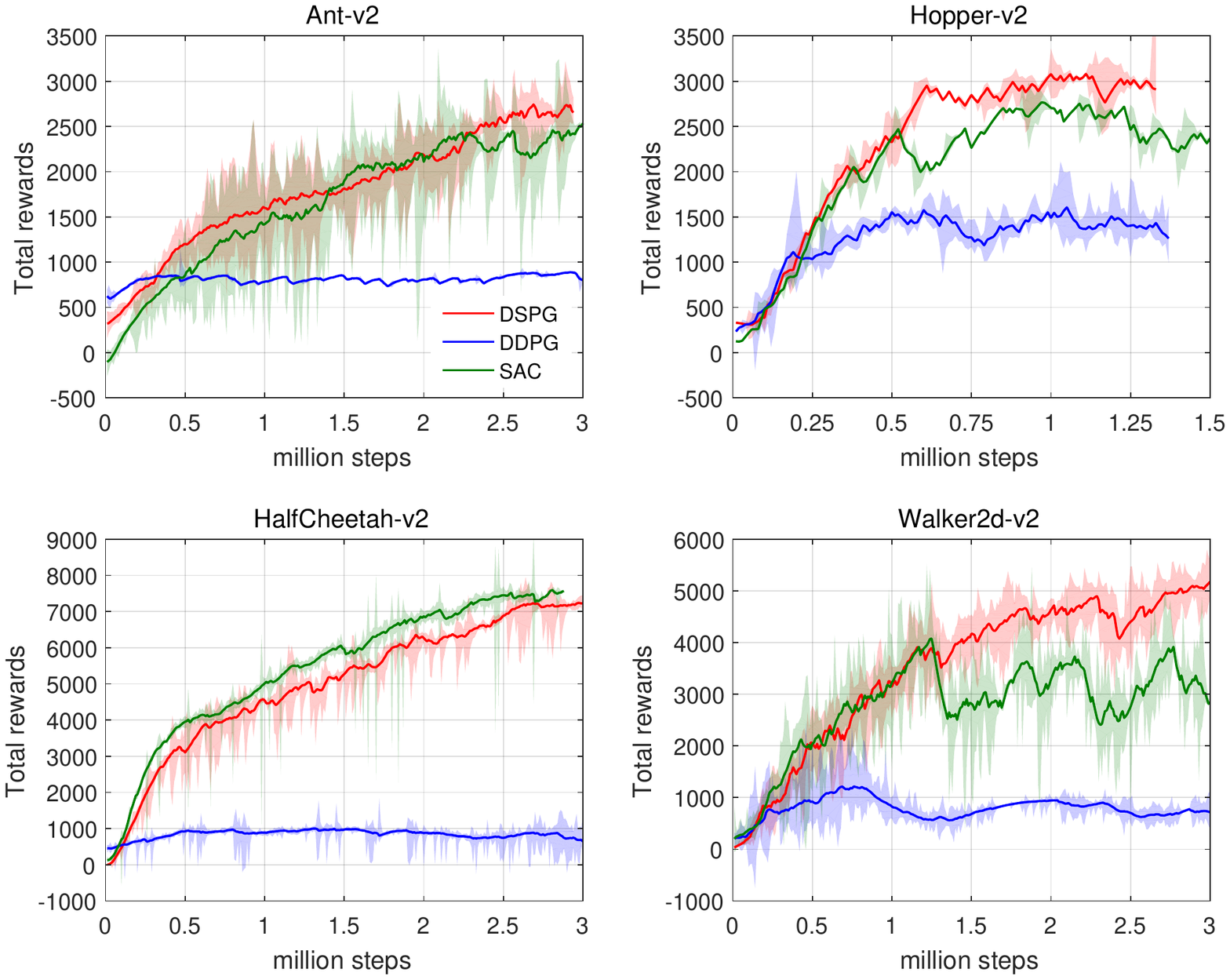}
 \caption{The results of DSPG against DDPG and SAC baselines on continuous control benchmarks. Each plot shows the average total rewards of evaluation rollouts across 7 randomly seeded training runs after choosing best hyperparameters. The x-axis shows millions of environment steps.}
 \label{fig:learning_curves}
\end{figure}

We present the average total rewards over training of DSPG, DDPG and SAC in Figure \ref{fig:learning_curves}. And Table \ref{tab:stats} shows the best average total rewards in all steps of training for our implementation and baselines. The results show that, overall, DSPG substantially outperforms DDPG on all of benchmark tasks, in terms of both sample efficiency and convergent performance. The notable gap of performance between DSPG and DDPG is ascribed to the hyperparameter sensitivity of DDPG, but also suggests significant advantage of our DSPG relative to DDPG in challenging tasks with very high state and action dimensionality. Moreover, it can be seen in Figure \ref{fig:learning_curves} that DSPG beats the performance of SAC in Hopper and Walker2d, and can match the performance of SAC in Ant and HalfCheetah.

In addition to the average total rewards, the stability of algorithm also plays a crucial role in performance. And a conclusion can be drawn from Figure \ref{fig:learning_curves} that DSPG outperforms consistently SAC in terms of the stability in all four benchmark tasks. The good stability of DSPG can be demonstrated from two observations. First, clipped soft policy gradient is applied to guarantee steepest optimization direction and prevent large update step of policy, which is typically considered to be main source of instability. Second, only one critic network is used to estimate soft Q-function with double sampling approach in DSPG, which avoids cumulative approximation errors resulting from a separate soft value function approximator. In practical, we suggest to take into account available computing resources and the cost of sampling when choosing appropriate algorithm for specific task. For tasks with limited computing resources but low cost of sampling, DSPG seems to be better.
\renewcommand\arraystretch{1.25}
\begin{table}
\centering
\begin{tabular}{lccc}
\hline
Domains         & DSPG                & SAC                 & DDPG        \\
\hline
Ant-v2          & 3384.074            & $\mathbf{3472.346}$ & 1012.242    \\
HalfCheetah-v2  & 7889.747            & $\mathbf{8103.202}$ & 1014.372    \\
Hopper-v2       & $\mathbf{3674.029}$ & 3652.893            & 3097.082    \\
Walker2d-v2     & $\mathbf{6060.884}$ & 5520.419            & 2507.199    \\
\hline
\end{tabular}
\caption{The results for best average total rewards in all training steps for our DSPG, DDPG and SAC implementations. These results are each on different setups with different hyperparameter searches. Thus, although it is not possible to make any definitive claims based on this data, we do conclude that our results are overall competitive with DDPG and SAC baselines.}
\label{tab:stats}
\end{table}

\section{Conclusion}
We presented deep soft policy gradient (DSPG), an off-policy actor-critic and model-free maximum entropy deep reinforcement learning algorithm. Our theoretical results derive soft policy gradient based on entropy regularized expected reward objective. Building on this result, we formulated our DSPG algorithm by combining soft policy gradient with soft Bellman equation. The experimental results suggest that DSPG can perform well on a set of challenging continuous control tasks, improving upon DDPG and SAC in terms of average total rewards and stability. Moreover, soft policy gradient provides a promising avenue for future works.

\section*{Acknowledgments}
This research was supported by the National Science Foundation of China under Grant 41427806.

\appendix
\section{Experimental Setup}\label{app:setup}
Throughout all experiments, we use Adam for learning the neural network parameters with a learning rate $5\times 10^{-5}$ and $5\times 10^{-4}$ for the actor and critic respectively. For critic we use a discount factor of $\gamma=0.99$. For the soft target updates we use $\alpha=0.01$. Both the actor and critic are represented by full-connected feed-forward neural network with two hidden layers of dimensions 512. And all hidden layers use ReLU activation. Specially, we use identity and sigmoid activations for the mean and standard deviation in the output layer respectively. The algorithm use a replay buffer size of three million and train with minibatch sizes of 100 for each train step. Training does not start until the replay buffer has enough samples for a minibatch and does not stop until the global time step equals to the threshold of $3\times 10^6$. In addition, we scales the reward function by a factor of 5 for all four tasks, as is common in prior works.

\newpage
\bibliographystyle{named}
\bibliography{ijcai19}

\end{document}